\documentclass{article}

\usepackage[final, nonatbib]{nips_2017}

\usepackage[utf8]{inputenc} 
\usepackage[T1]{fontenc}    
\usepackage{hyperref}       
\usepackage{url}            
\usepackage{booktabs}       
\usepackage{nicefrac}       
\usepackage{microtype}      
\usepackage{amsmath,amsthm,amssymb,amsfonts}
\usepackage[pdftex]{graphicx}
\usepackage{subcaption}
\usepackage{float}
\usepackage{amsthm}

\newtheorem{theorem}{Theorem}
\newtheorem{corollary}{Corollary}[theorem]
\newtheorem{lemma}{Lemma}
\newtheorem{conj}{Conjecture}

\newcommand\numberthis{\addtocounter{equation}{1}\tag{\theequation}}

\title{Convergence Analysis of Gradient Descent Algorithms with Proportional Updates}

\author{
  Igor Gitman (igitman) \\
  \texttt{igitman@andrew.cmu.edu} \\
  \And
  Deepak Dilipkumar (ddilipku) \\
  \texttt{ddilipku@andrew.cmu.edu} \\
  \And
  Ben Parr (bparr) \\
  \texttt{bparr@cmu.edu} \\
}

\newcommand{\R}{\mathbb{R}}

\newcommand{\inner}[2]{\left\langle#1,#2 \right\rangle}
\newcommand{\rbr}[1]{\left(#1\right)}
\newcommand{\sbr}[1]{\left[#1\right]}
\newcommand{\cbr}[1]{\left\{#1\right\}}
\newcommand{\nbr}[1]{\left\|#1\right\|}
\newcommand{\abr}[1]{\left|#1\right|}

\DeclareMathOperator*{\argmin}{arg\,min}

\begin{document}

\maketitle


\section{Introduction}

The rise of deep learning in recent years has brought with it increasingly clever optimization methods to deal with complex, non-linear loss functions~\cite{ruder}. These methods are often designed with convex optimization in mind, but have been shown to work well in practice even for the highly non-convex optimization associated with neural networks. However, one significant drawback of these methods when they are applied to deep learning is that the magnitude of the update step is sometimes disproportionate to the magnitude of the weights (much smaller or larger), leading to training instabilities such as vanishing and exploding gradients~\cite{pascanu2013difficulty}. An idea to combat this issue is gradient descent with proportional updates.

Gradient descent with proportional updates was introduced in 2017. It was independently developed by You et al~\cite{you2017large} (Layer-wise Adaptive Rate Scaling (LARS) algorithm) and by Abu-El-Haija~\cite{abu2017proportionate} (PercentDelta algorithm). The basic idea of both of these algorithms is to make each step of the gradient descent proportional to the current weight norm and independent of the gradient magnitude:
\begin{equation}\label{eq:propupdates}
w_{k+1} = w_k - \lambda_k \frac{\nbr{w_k}}{\nbr{\nabla L(w_k)}}\nabla L(w_k) 
\end{equation}

where $L(w)$ is the objective function being optimized. $\nbr{.}$ is $L_2$ norm for LARS and $L_1$ norm for PercentDelta, and both of the algorithms apply this update independently for each layer. As mentioned before, this attempts to alleviate the fact that the magnitude of the gradients is not a reliable step size indicator by replacing it with the magnitude of the current weight vector, leading to proportional updates.

It is common in the context of new optimization methods to prove convergence or derive regret bounds under the assumption of Lipschitz continuity and convexity~\cite{proofs}. However, even though LARS and PercentDelta were shown to work well in practice, there is no theoretical analysis of the convergence properties of these algorithms. Thus it is not clear if the idea of gradient descent with proportional updates is used in the optimal way, or if it could be improved by using a different norm or specific learning rate schedule, for example. Moreover, it is not clear if these algorithms can be extended to other problems, besides neural networks.

In this project we attempt to answer these questions by establishing the theoretical analysis of gradient descent with proportional updates, and verifying this analysis with empirical examples.


\section{Related Work}

The flurry of activity in deep learning in recent years was kicked off by the results presented in Krizhevsky et al \cite{alexnet}. Since then, most research done to improve fundamental deep learning techniques has typically focused on two areas - changes to the network architecture \cite{zfnet}\cite{inceptionnet}\cite{resnet}\cite{densenet},  and changes to the training mechanism. The latter includes improved regularization \cite{dropout}, initialization schemes \cite{glorot} and optimization techniques, which are known to play a significant role in the performance of deep networks in practice \cite{impmomentum}.
\\
\\
While AlexNet achieved excellent results using Stochastic Gradient Descent with momentum, SGD is known to face difficulties when optimizing certain types of problems \cite{longtermdifficult}. It is also known to perform poorly with badly tuned learning rates or decay schedules \cite{peskylr}. A number of gradient-based optimization methods have been proposed to remedy this by using adaptive learning rates. These include RMSProp~\cite{tieleman2012lecture}, AdaGrad~\cite{adagrad}, AdaDelta~\cite{adadelta} and ADAM~\cite{adam}. Although these techniques have interesting theoretical properties and often improve training stability for deep architectures, in practice the most popular optimization method remains SGD with momentum, as it tends to outperform these other methods in terms of accuracy given enough training epochs.
\\
\\
As explained earlier, the papers most relevant for us are \cite{you2017large} and \cite{abu2017proportionate} which introduced the idea of gradient descent with proportional updates. These methods have an additional advantage over traditional adaptive learning rate methods like those mentioned above in that the learning rate is computed per layer instead of per weight, leading to improved stability.

We now formalize and test the proportional update method LARS, which assumes that $L_2$ norm is used in the update rule~\ref{eq:propupdates}\footnote{Most of the presented results can be extended to the $L_1$ norm case using the equivalence of the norms in finite-dimensional spaces.}. We start by analyzing functions operating in 1-dimensional space, since the convergence analysis is greatly simplified in that case. Using the insights found from the 1D derivations we present the convergence analysis of LARS for general functions $f: \R^n \to \R$ in section~\ref{sec:general-analysis}.


\section{Convergence analysis of 1D functions}\label{sec:1d-analysis}
\subsection{Simple quadratic in 1D}
To gain an intuition into the performance of LARS algorithm, we will start by analyzing the simplest quadratic case in 1-dimensional space: $f(w) = \frac{1}{2}(w - a)^2$ where $w, a \in \R$. In this analysis we assume that $\eta$ is constant throughout iterations, yielding the following update rule:
\[ w_{k+1} = w_k - \eta \frac{\abr{w_k}}{\abr{\nabla f(w_k)}}\nabla f(w_k) = w_k - \eta\abr{w_k}\text{sign}(w_k - a)  \]

Assuming $a > 0$ and $w_0 > 0$\footnote{The same analysis can be done for the case when $a < 0$ and $w_0 < 0$. However, when the optimal and initial points have different signs, LARS will fail to converge, which is discussed in the next section.} we can say that $w_k > 0\ \forall k \Rightarrow \abr{w_k} = w_k$. Then, let's notice that if 
$w_k \in [a(1 - \eta); a(1 + \eta)]$ then $w_t \in [a(1 - \eta); a(1 + \eta)]\ \forall\ t > k$. Indeed if 
\[ w_k \in \left[a(1 - \eta); a\right) \Rightarrow a(1 - \eta) \le w_k <  w_k + \eta w_k = w_{k+1} = (1+\eta)w_k < (1 + \eta)a \]
The same arguments hold when $w_k \in (a, a(1 + \eta)]$. If $w_k = a \Rightarrow w_{k+1} = a$, since $\text{sign}(w_k - a) = 0$. Applying this inductively for all $t > k$ we conclude the proof.

Assuming $w_0 < a$ we can calculate how many iterations are needed to reach the $[a(1 - \eta); a(1 + \eta)]$ interval. Indeed, for all $k$, such that $w_k < a$ we get
\begin{align*} 
    w_k = (1 + \eta)w_{k-1} = \cdots = (1 + \eta)^kw_0 &\ge a(1 - \eta) \Leftrightarrow k\log\rbr{1 + \eta} \ge \log\rbr{\frac{a(1 - \eta)}{w_0}} \Leftrightarrow \\ \Leftrightarrow\ k &\ge \log\rbr{\frac{a(1 - \eta)}{w_0}} \Bigg/ \log\rbr{1 + \eta} \approx 
    \frac{1}{\eta}\log\rbr{\frac{a}{w_0}}
\end{align*}

Moreover, $w_k$ can't skip the $[a(1 - \eta); a(1 + \eta)]$ interval, since if 
\[ w_k < a(1 - \eta) < a \Rightarrow w_{k + 1} = (1 + \eta)w_k < a(1+\eta) \]
Similar analysis can be carried out for the $w_0 > a$ case to obtain:
\begin{align*} 
    k \ge -\log\rbr{\frac{a(1 + \eta)}{w_0}} \Bigg/ \log\rbr{1 - \eta} \approx \frac{1}{\eta}\log\rbr{\frac{a}{w_0}}
\end{align*}

Thus we've shown that $w_k$ will reach the desired interval with $\approx O\rbr{\frac{1}{\eta}}$ number of steps and will stay in that interval ever after. However, it can get exactly to the optimum from only 2 specific points and if it doesn't reach the optimum exactly, it will keep oscillating in that interval without any convergence. Thus in order to reach $\epsilon$-convergence we need to make this interval $\epsilon$-small:
\[ \max\cbr{a - a(1 - \eta), a(1 + \eta) - a} = \max\cbr{\eta a, \eta a} = \eta a < \epsilon \Rightarrow \eta < \frac{\epsilon}{a} \]

So, in this simple case, LARS algorithm has sublinear convergence speed $O\rbr{\frac{a}{\epsilon}}$, similar to the usual gradient descent for convex functions with Lipschitz gradient. However, in this case there is an undesired linear dependence on the norm of the solution.
\subsection{Origin as attractive fixed point}\label{why_implementation_is_sgd_sometimes}
The analysis in the previous section was carried out with the assumption that $a$ and $w_0$ have the same signs. If that is not true, it is easy to see that LARS algorithm will fail to converge. Indeed, assuming without loss of generality that $w_0 < 0$ and $a > 0$, we have
\[ 
w_{k+1} = w_k - \eta\abr{w_k}\text{sign}(w_k - a) = w_k - \eta w_k = (1 - \eta)w_k \xrightarrow[k \to \infty]{} 0
\]

Which means that the origin is a fixed attractive point of the iterates $w_k$ for this situation. Intuitively, it is not surprising that LARS iterates will never change sign, since the next point is proportional to the current point with coefficient $\eta < 1$, making $w_k$ smaller and smaller when moving in the direction of origin.

You et al~\cite{you2017large} overcome this problem by switching to the usual stochastic gradient descent in the vicinity of the origin (i.e. when $\nbr{w_k}_2 < \beta$ for some predefined $\beta$). For now we follow the same strategy in our experiments, although modifying LARS to avoid this problem in a more natural way is a direction of future research.

\subsection{General function in 1D}\label{sec:1d-general}
Somewhat surprisingly the analysis presented in the previous two sections can be directly applied to any differentiable convex function $f : \R \to \R$. Indeed, the LARS updates only depend on the sign of the gradient, thus, for any function with $a = \argmin_w f(w)$ we will get the same update equation
\[ w_{k+1} = w_k - \eta \frac{\abr{w_k}}{\abr{f'(w_k)}}f'(w_k) = w_k - \eta\abr{w_k}\text{sign}\rbr{f'(w)} = w_k - \eta\abr{w_k}\text{sign}(w_k - a) \]
since for any convex differentiable function the following property holds
\[ \rbr{f'(x) - f'(y)}\rbr{x - y} \ge 0\ \forall x, y \Rightarrow f'(w)\rbr{w - a} \ge 0 \Rightarrow \text{sign}\rbr{f'(w)} = \text{sign}(w - a) \] 

This suggests that for the general function $f: \R^n \to \R$ the convergence analysis can be carried out without assuming anything except convexity and differentiability. In particular, it suggests that the analysis presented in the next sections can be improved by relaxing the Lipschitz gradient assumption.
\section{General convex function with Lipschitz gradient}\label{sec:general-analysis}
In this section we will present the convergence analysis for the general convex functions with Lipschitz gradient. As we have seen from the 1-dimensional case, LARS algorithm with fixed learning rate can only converge to a region around the optimal point, which depends on the norm of the optimal solution $w^*$ and learning rate $\eta$. We will aim to establish a similar result for the general convex differentiable function with Lipschitz gradient. We will first prove that LARS updates never diverge to infinity, given that the learning rate is set to be sufficiently small. Using that result we will show that the size of the convergence region becomes smaller with the decrease of $\eta$ and thus $\epsilon$-convergence can be guaranteed by setting $\eta$ to be small enough. Finally, we will show that LARS algorithm has at least sublinear convergence rate $O(1/\epsilon^{1.5})$ with fixed learning rate and that it converges under the standard assumptions on the learning rate decay schedule. It might be possible to improve the established convergence rate by using more precise lower bounds.

\subsection{Basic convergence proof}
In this section we will show that given a sufficiently small learning rate, LARS iterates never diverge to infinity, meaning that LARS always converges to some region around the optimal point. This statement is very intuitive, since in order for LARS to diverge to infinity, it has to either follow mostly the same direction for infinite number of iterations, or alternate directions, but have infinitely increasing step sizes. The first one is not possible, since the update direction always has acute angle with the direction towards optimum. Thus, after LARS ``passes'' the optimum it will have to ``reverse'' its direction. The step sizes should stay finite since the norm of the next point is proportional to the norm of the current point and cannot be consistently increasing given that LARS periodically has to ``reverse'' the update direction.

Even though this statement is intuitively clear, we were not able to explicitly construct the convergence region as was possible in the 1-dimensional case. Instead, we conducted a non-constructive proof by showing that the set of points, for which the next LARS step is farther away from the optimum than the current step, is bounded. Formally, we proved the following Lemma:

\begin{lemma}\label{lm:set}
    For any strongly convex function $f(w)$ with Lipschitz gradient, the following set 
    \[ S = \cbr{w: \nbr{w - \eta\frac{\nbr{w}_2}{\nbr{\nabla f(w)}_2}\nabla f(w) - w^*}_2 > \nbr{w - w^*}_2} \]
    is bounded, i.e. $\exists M < \infty: \sup_{w \in S}\nbr{w}_2 \le M$, providing $\eta$ is small enough, $w^* = \argmin_w f(w)$
\end{lemma}

The only reason why we need strong convexity is to show that the direction towards the optimum and the current update direction cannot become orthogonal when $\nbr{w_k}_2 \to \infty$. It seems intuitive that some similar results should hold for the general case, however we were not able to prove it directly. We state the following conjecture

\begin{conj}\label{conj}
    For any convex function $f(w)$ with Lipschitz gradient
    \[ \exists\ C > 0: \cos{\alpha} \equiv \frac{\inner{w - w^*}{\nabla f(w) -\nabla f(w^*) }}{\nbr{w - w^*}_2\nbr{\nabla f(w) - \nabla f(w^*)}_2} = \frac{\inner{w - w^*}{\nabla f(w)}}{\nbr{w - w^*}_2\nbr{\nabla f(w)}_2} \ge C \ \forall w \]
\end{conj}

which, if true, will mean that all the following results hold without the strong convexity assumptions. We conclude this section with the following theorem

\begin{theorem}\label{th:main-conv}
    For any strongly convex function (convex if Conjecture~\ref{conj} is true) with Lipschitz gradient $f(w)$, LARS sequence of updates with fixed learning rate $w_{k+1} = w_k - \eta\frac{\nbr{w_k}_2}{\nbr{\nabla f(w_k)}_2}\nabla f(w_k)$ does not diverge to infinity, i.e. 
    \[ \exists M > 0\ \text{such that}\ \nbr{w_k - w^*}_2 \le M \ \forall k \]
    where $w^* = \argmin_w f(w)$ and $\eta$ is chosen small enough.
\end{theorem}

\subsection{Rate of convergence with fixed learning rate}
As we have seen in section~\ref{sec:1d-analysis}, LARS algorithm does not converge if the iterative sequence needs to get close to the origin. Thus, in order to show that the convergence region shrinks with the decrease of the learning rate $\eta$ we will need to assume that such situation never happens. Given this assumption and the result of Theorem~\ref{th:main-conv} we can prove the following

\begin{theorem}\label{th:rate-fixed}
    Suppose that function $f : \R^n \to \R$ is strongly convex (convex if Conjecture 1 is true) and differentiable, and that its gradient is Lipschitz continuous with constant $L > 0$.
    Under the additional assumption that $\exists M_2 > 0 : \min_k\nbr{w_k}_2 \ge M_2$, the LARS sequence of updates with fixed learning rate $\eta$ satisfies the following inequality for each $k$
    \[ \min_k\sbr{f(w_k)} - f(w^*) \le \frac{C_1}{k\eta} + \eta^2 C_2 \]
    with
    \[ C_1 = \frac{LM_1}{2M_2}, C_2 = \frac{L\rbr{\nbr{w^*}_2 + M_1}^2}{2} \]
    with $f(w^*) = \min_w f(w)$ and $M_1$ being the constant from Theorem~\ref{th:main-conv}. 
\end{theorem}

Notice that the first term in the bound goes to zero, which means that LARS iterates converge to the region around the optimum with the size shrinking with the decrease of $\eta$. Thus, making $\eta$ small enough and keeping track of the minimal value of $f(w_k)$ we can achieve $\epsilon$-convergence which is summarized in the following

\begin{corollary}\label{cor}
    Under the conditions of Theorem~\ref{th:rate-fixed}, LARS has at least a sublinear $\epsilon$-convergence rate of $O\rbr{\frac{1}{\epsilon^{1.5}}}$, when $\eta = \sqrt{\frac{\epsilon}{3C_2}}$.
\end{corollary}
\begin{proof}
The result follows directly by substituting $\eta$ and $\epsilon$ in the bound:
\begin{align*} 
    \min_k\sbr{f(w_k)} - f(w^*) \le \frac{C_1}{k\eta} + \eta^2 C_2 < \epsilon \Rightarrow k > \frac{C_1}{\eta\rbr{\epsilon - \eta^2C_2}}
\end{align*}
Minimizing the right side under the constraint $\eta^2 C_2 < \epsilon$ we obtain the optimal $\eta = \sqrt{\frac{\epsilon}{3C_2}}$. Then, to reach $\epsilon$-convergence $k$ should be
\[ k > \frac{C_1}{\eta\rbr{\epsilon - \eta^2C_2}} = \frac{3\sqrt{3C_2}C_1}{2\epsilon\sqrt{\epsilon}} = O\rbr{\frac{1}{\epsilon^{1.5}}} \]

\end{proof}

\subsection{Convergence with decaying learning rate}
Finally, let's establish the convergence when $\eta \equiv \eta_k \to 0$. The following result holds
\begin{theorem}\label{th:rate-decay}
    Under the assumptions of Theorem~\ref{th:rate-fixed}, LARS converges to the optimal value with any decaying learning rate sequence $\eta_k$ satisfying 
    \[ \sum_{k=1}^{\infty}\eta_k = \infty, \sum_{k=1}^{\infty}\eta_k^2 < \infty \]
\end{theorem}

Even though this result establishes the convergence under decaying learning rate schedule, it cannot be used to improve the convergence rate established in Corollary~\ref{cor}. With the current analysis technique LARS with diminishing learning rates can only be shown to converge with $O\rbr{\frac{1}{\epsilon^2}}$ rate, which is worse than what is established in the previous section. However, the analysis can likely be improved, since, for example, we ignore the fact that $\nabla f(w_k)$ converges to zero, which can significantly tighten the bounds.

\begin{figure}[t]
    \centering
    \begin{subfigure}{0.4\textwidth}
        \centering
        \includegraphics[width=1.0\textwidth]{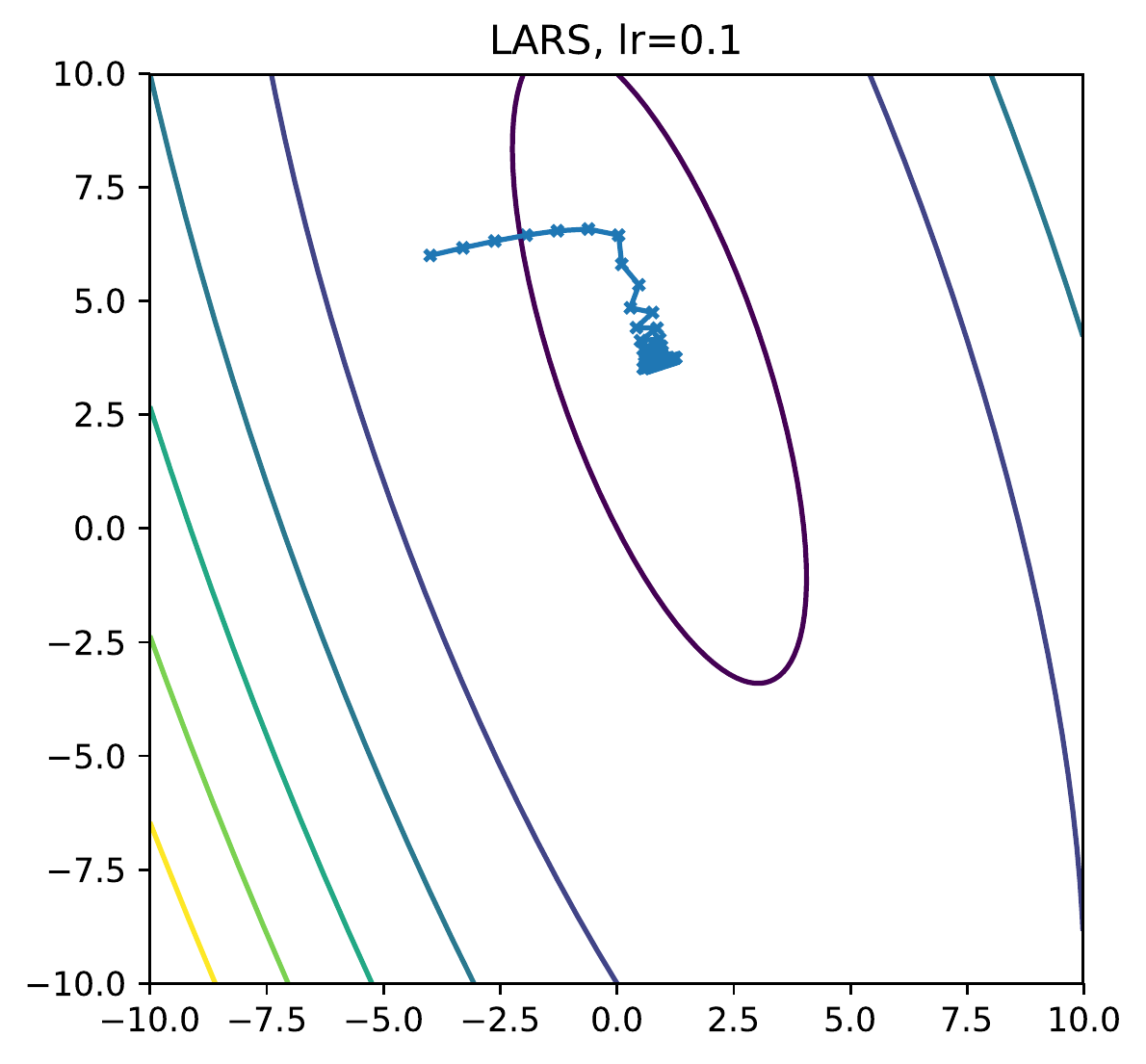}
        \caption{}
    \end{subfigure}
    \begin{subfigure}{0.4\textwidth}
        \centering
        \includegraphics[width=1.0\textwidth]{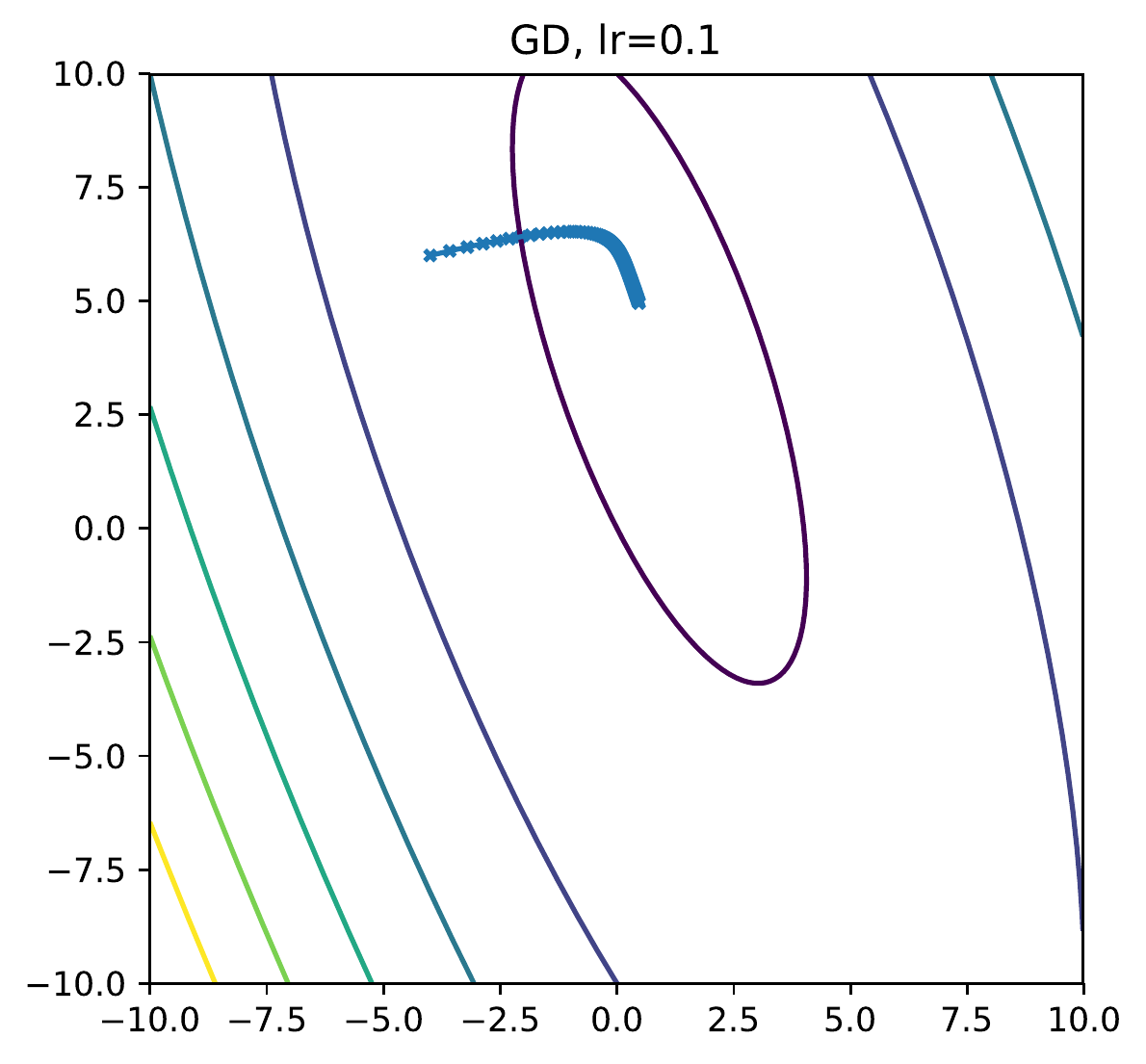}
        \caption{}
    \end{subfigure}
    \begin{subfigure}{0.4\textwidth}
        \centering
        \includegraphics[width=1.0\textwidth]{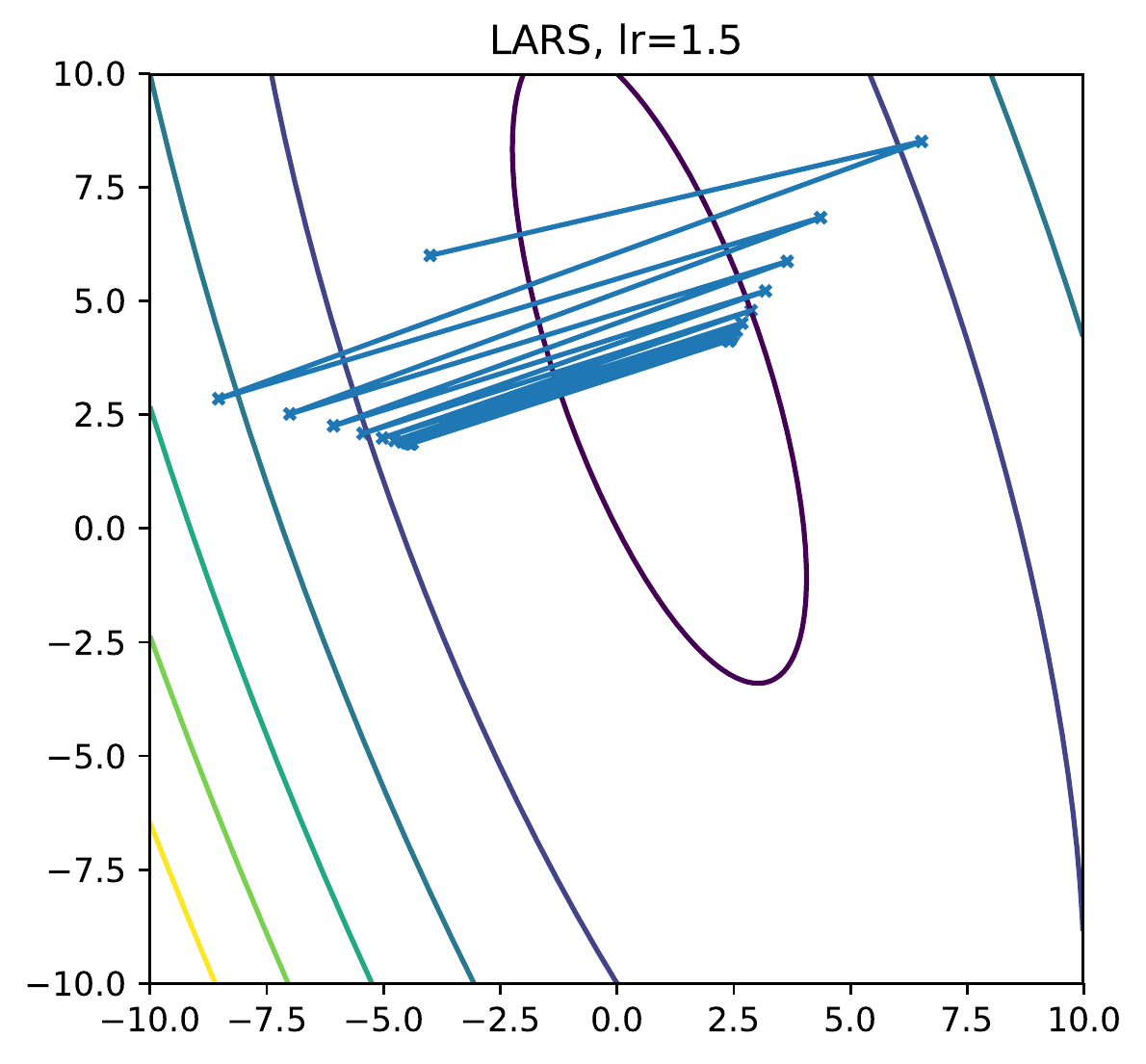}
        \caption{}
    \end{subfigure}
    \begin{subfigure}{0.4\textwidth}
        \centering
        \includegraphics[width=1.0\textwidth]{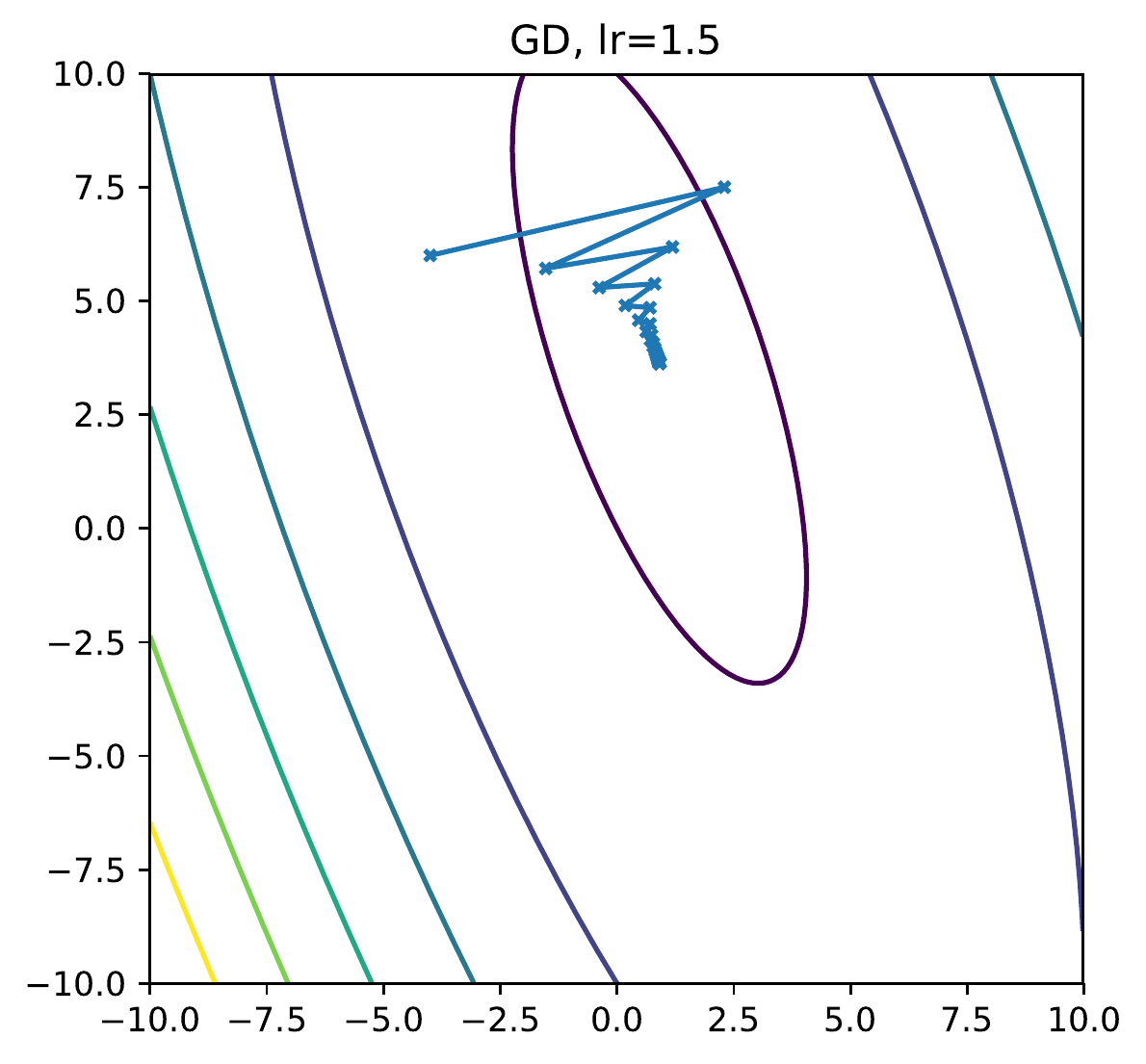}
        \caption{}
    \end{subfigure}
    \caption{Simple experiments comparing LARS with fixed learning rate and Gradient Descent (GD) with fixed learning rate. Plots (a) and (b) show that LARS can converge faster than GD, since each update is approximately the same size. However, it only converges to a certain region around the optimal point. Plots (c) and (d) show that this convergence region increases with the increase of learning rate. It is also clear that it is nor symmetric around the optimal point which is consistent with the analysis presented in section~\ref{sec:1d-analysis}.}
    \label{fig:2dexps}
\end{figure}

\section{Discussion}
Let's summarize the obtained results. We have shown that the LARS algorithm with fixed learning rate converges to some area around the optimal value that can be made arbitrarily small by decreasing the learning rate. To provide an empirical justification for this claim we conducted a simple 2-dimensional experiment with a quadratic function. The results of comparing LARS and gradient descent are presented in Figure~\ref{fig:2dexps}. To guarantee an $\epsilon$-convergence, the learning rate needs to be decreased so that the convergence region becomes $\epsilon$ small. Following this strategy we were able to prove that LARS converges with at least the rate of $O\rbr{\frac{1}{\epsilon^{1.5}}}$, which is worse than the convergence rate of the usual gradient descent under the same conditions. However, the given convergence rate might be improved, since we have shown that at least for 1-dimensional functions, LARS enjoys faster convergence rate of $O\rbr{\frac{1}{\epsilon}}$. One way to do that would be to find tighter inequalities than the ones used in Theorems~\ref{th:rate-fixed},~\ref{th:rate-decay}. 

Another potential advantage of LARS over the gradient descent is that it might be possible to show the convergence under less strict assumptions, since in 1-D case the algorithm will perform exactly the same iterates for any convex differentiable function with the same position of the minimum. Thus, likely, the Lipschitz gradient assumption in the aforementioned analysis can be relaxed and that might imply that LARS should perform better than the gradient descent on the ill-conditioned problems.

We confirm the aforementioned claims by experimentally showing that LARS performs comparable to the other optimization algorithms on three simple benchmarks.


\section{Experiments}

\subsection{Implementation Details}

We have implemented LARS in TensorFlow using the standard optimizer extension. Specifically, we reuse the built-in \texttt{MomentumOptimizer} and apply the LARS-specific logic using the \texttt{opt.apply\_gradients()} method. In TensorFlow, the weights and biases of a network are represented as TensorFlow variables. In LARS, each TensorFlow variable is given its own \texttt{MomentumOptimizer} so each momentum is maintained independently, while still only computing the gradients of each variable once per update.

We use two additional tweaks, not mentioned in the LARS Algorithm 1 ~\cite{you2017large}. First, we introduced a small constant epsilon term in the denominator of the local learning rate calculation, which improves numerical stability when close to the optimal value. Second, if the norm of the variable is close to zero (less than 0.01), then we switch to SGD, as discussed in Section \ref{why_implementation_is_sgd_sometimes}.

\subsection{Convex Problems}

In order to experimentally examine the theoretical results that we derived in previous sections, we run full-batch versions of LARS, GD and ADAM on two convex problems - SVM \cite{svm} and logistic regression. We also compare these to the stochastic (mini-batch) variants of the algorithms, as these are more common in practice.

\subsubsection{Support Vector Machine}

\begin{figure}
    \centering

    \begin{subfigure}{0.45\textwidth}
        \centering
        \includegraphics[width=1.0\textwidth]{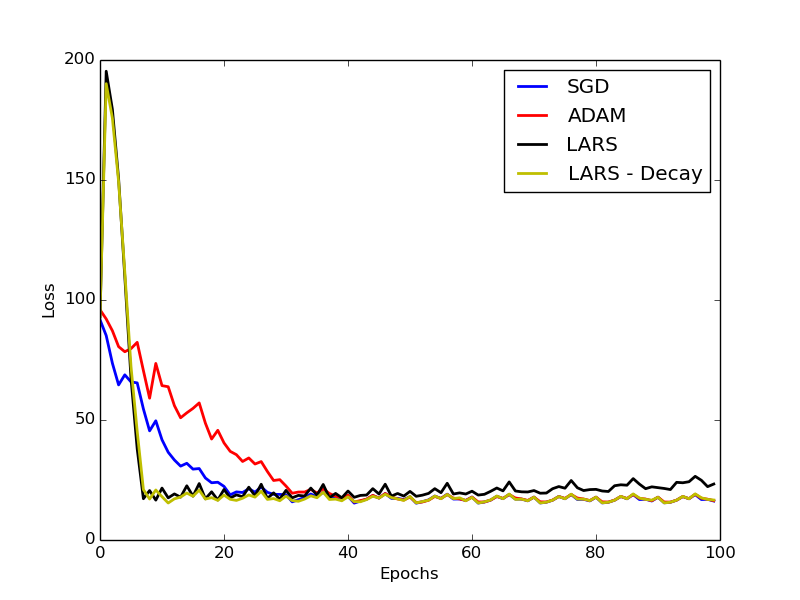}
        \caption{SVM - Mini-Batch}
    \end{subfigure}
    \begin{subfigure}{0.45\textwidth}
        \centering
        \includegraphics[width=1.0\textwidth]{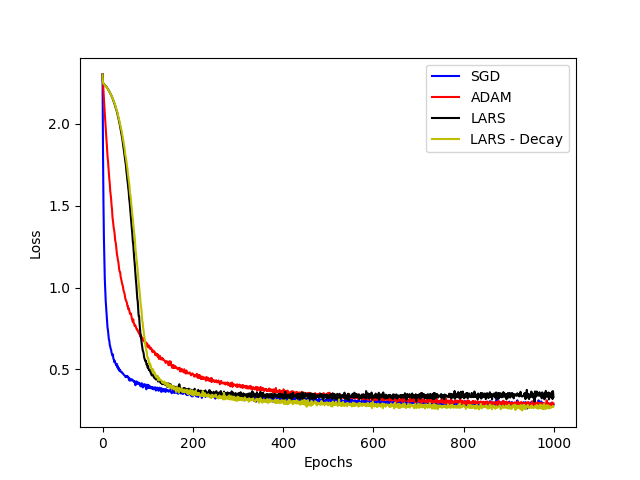}
        \caption{Logistic Regression - Mini-Batch}
    \end{subfigure}
    \begin{subfigure}{0.45\textwidth}
        \centering
        \includegraphics[width=1.0\textwidth]{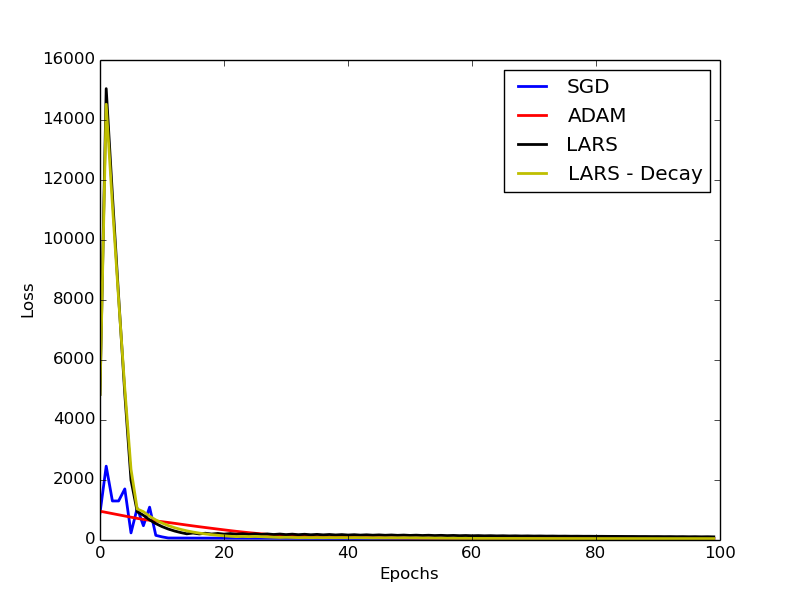}
        \caption{SVM - Full Batch}
    \end{subfigure}
    \begin{subfigure}{0.45\textwidth}
        \centering
        \includegraphics[width=1.0\textwidth]{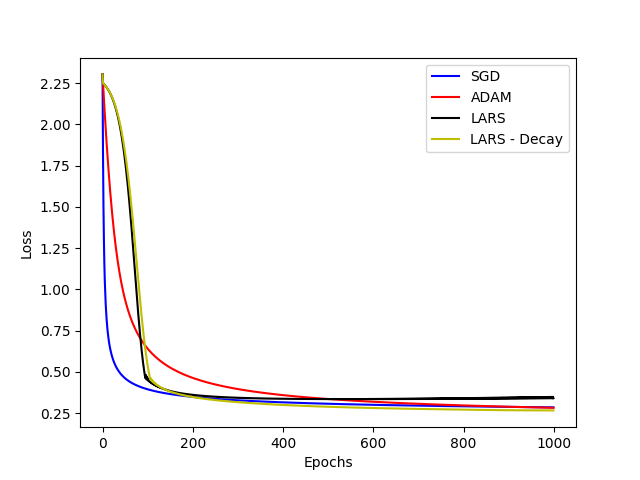}
        \caption{Logistic Regression - Full Batch}
    \end{subfigure}
  \caption{Comparison of the convergence of SGD, ADAM and LARS on two convex problems. ``LARS - Decay'' refers to LARS with a decaying learning rate schedule. Note than 1 epoch indicates training conducted on a single batch.}
  \label{fig:convex}
\end{figure}


This problem involves simple, linearly separable data that is tackled using a linear SVM. The SVM hinge loss is minimized using Tensorflow's GradientDescentOptimizer and AdamOptimizer, along with our custom Tensorflow implementation of LARS. We run experiments on both a mini-batch and full-batch version of all 3 optimizers. Note that for the mini-batch version, batches of size 100 (for a dataset of 1000 training point) are used. Some tuning is done to decide on the learning rates for the three algorithms, with SGD using 0.01, ADAM using 0.1 and LARS using 0.1. The graph comparing convergence of the three algorithms is shown on the left side of Fig. \ref{fig:convex}.
\\
\\
Note that while all three algorithms converge to the same final optimum (which is to be expected given the convexity of the problem), LARS tends to converge much quicker than the other two. The initial spike in the loss is likely due to the initialized weights having a relatively high norm. The convergence pattern is similar across full-batch and mini-batch versions, with the former having lower variance as is to be expected. This also indicates that the analysis for full-batch as well as for stochastic variants of LARS can be improved to match the convergence rates of the standard optimization methods.
\\
\\
Also note that LARS appears to oscillate near the optimum towards the end of training. As described earlier, this is due to the fact that LARS can only guarantee convergence to some area around the optimum, dependant on the learning rate. If we use a decaying learning rate schedule that starts at 0.1 and goes to 0 at the end of training, we find that LARS convergences to exactly the same optimum as SGD and ADAM, as indicated by the yellow line in Fig. \ref{fig:convex}.

\subsubsection{Logistic Regression}

The second convex problem that we experimented with was logistic regression, trained to classify MNIST digits. Both full-batch and mini-batch versions are tested again. For mini-batch versions, we use a large batch size of 10,000 (the training data size is 60,000). We tuned the learning rates for the three algorithms, with SGD, ADAM and LARS using 0.5, 0.001 and 0.05 respectively. The results are shown on the right side of Fig. \ref{fig:convex}. We see that LARS appears to outperform ADAM in terms of epochs needed to reach the optimum, but it is slower to converge than SGD. Once again all three converge to the same final optimum.
\\
\\
Similar to the SVM case, we see that LARS with a fixed learning rate appears to be deviating away from the optimum towards the end of training. This is again fixed by using a decaying learning rate schedule. These examples illustrate that decaying the learning rate plays more of a vital role for the performance of LARS than for traditional optimizers like SGD and ADAM.

\subsection{Non-Convex Problem}
TensorFlow includes a ResNet v2 model~\cite{he2016identity} that trains on the CIFAR-10 dataset\footnote{https://github.com/tensorflow/models/tree/master/official/resnet}.The only change we made to the TensorFlow code was to swap the optimizer. Each ResNet model was trained on a dedicated NVIDIA TITAN X (Pascal) GPU with 12 GB of memory. All three optimizers achieved similar average training throughput. A single train step took on average 44.8 milliseconds for SGD,  42.6 milliseconds for ADAM, and 48.0 milliseconds for LARS. Both SGD and LARS performed similarly, with LARS achieving a slightly better performance (1.3329\% Train Error for LARS, and  1.4584\% Train Error for SGD). ADAM initially performed best for the first 3000 train steps, but then fell behind both SGD and LARS after that. The loss function for each optimizer saw similar trends.

\begin{figure}[H]
    \centering
    \begin{subfigure}{0.48\textwidth}
        \centering
        \includegraphics[width=1.0\textwidth]{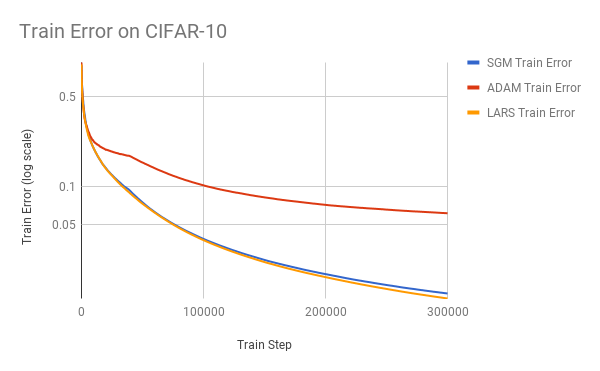}
        \caption{Train Error (log scale).}
    \end{subfigure}
    \begin{subfigure}{0.48\textwidth}
        \centering
        \includegraphics[width=1.0\textwidth]{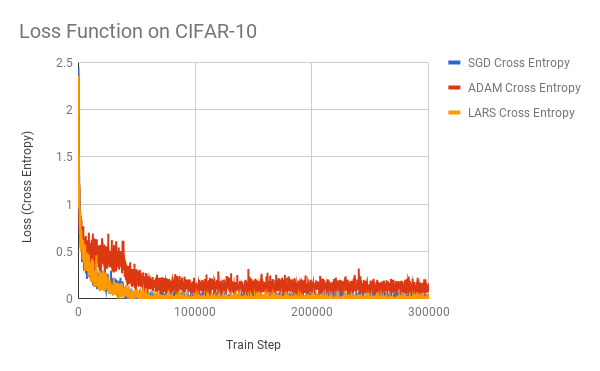}
        \caption{Train Loss (Cross Entropy).}
    \end{subfigure}
  \caption{Train error and loss function comparison of SGD, ADAM and LARS for ResNet v2 model on the CIFAR-10 dataset, with tuned initial learning rates of 0.1, 0.005, and 0.001, respectively. All optimizers use a batch size of 128.}
  \label{cifar10}
\end{figure}

\section{Conclusion}
We proved new theoretical results for gradient decent algorithms with proportional updates. Specifically, we showed that LARS has at least a sublinear convergence rate of $O \left( \frac{1}{\epsilon^{1.5}} \right)$. We have also demonstrated that LARS is guaranteed to converge under the standard decaying learning rate schedule assumptions. Even though, our theoretical results indicate that LARS might converge slower than other methods, empirically we observe the reverse effect for both stochastic and full-batch cases. Thus, our theoretical derivations can likely be improved by finding tighter bounds or relaxing certain assumptions. In the future, we hope to accomplish this, as well as analyze the other variants of LARS involving stochastic updates and momentum updates.

We would like to thank Dr. Aarti Singh and Dr. Pradeep Ravikumar for teaching Convex Optimization. Finally, we would like to thank Yang You and Boris Ginsburg from NVIDIA for their initial work on LARS.

\bibliography{main}
\bibliographystyle{abbrv}

\newpage
\section*{Appendix}
Proofs of the theoretical results.

\textbf{Lemma~\ref{lm:set}}.\textit{
    For any strongly convex function $f(w)$ with Lipschitz gradient, the following set 
    \[ S = \cbr{w: \nbr{w - \eta\frac{\nbr{w}_2}{\nbr{\nabla f(w)}_2}\nabla f(w) - w^*}_2 > \nbr{w - w^*}_2} \]
    is bounded, i.e. $\exists M < \infty: \sup_{w \in S}\nbr{w}_2 \le M$, providing $\eta$ is small enough, $w^* = \argmin_w f(w)$
}

\begin{proof}
    Let's consider a point $w \notin S, w \neq w^*$:
    \begin{align*}
        &\nbr{w - \eta\frac{\nbr{w}_2}{\nbr{\nabla f(w)}_2}\nabla f(w) - w^*}_2 \le \nbr{w - w^*}_2 \Leftrightarrow \numberthis \label{eq:set_rev} \\ \Leftrightarrow
        &\nbr{w - \eta\frac{\nbr{w}_2}{\nbr{\nabla f(w)}_2}\nabla f(w) - w^*}^2_2 \le \nbr{w - w^*}^2_2 \Leftrightarrow \\ \Leftrightarrow\ 
        &\eta^2\nbr{w}_2^2 - 2\eta\frac{\nbr{w}_2}{\nbr{\nabla f(w)}_2}\inner{w - w^*}{\nabla f(w)} \le 0 \Leftrightarrow \\ \Leftrightarrow
        &\eta\nbr{w}_2 - 2\nbr{w - w^*}_2\frac{\inner{w - w^*}{\nabla f(w)}}{\nbr{w - w^*}_2\nbr{\nabla f(w)}_2} \le 0 \Leftrightarrow \\ \Leftrightarrow
        &\frac{\nbr{w}_2}{\nbr{w - w^*}_2} \le \frac{2}{\eta}\frac{\inner{w - w^*}{\nabla f(w)}}{\nbr{w - w^*}_2\nbr{\nabla f(w)}_2} \equiv \frac{2}{\eta}\cos{\alpha} \numberthis \label{eq:set_bound}
    \end{align*}
    where $\alpha$ is the angle between $w - w^*$ and $\nabla f(w)$. Notice that when $f(w)$ is strongly convex with constant $m$ and has Lipschitz gradient with constant $L$, we can derive a lower bound on $\cos{\alpha}$:
    \begin{align*}
        m\nbr{w - w^*}_2^2 &\le \inner{w - w^*}{\nabla f(w) - \nabla f(w^*)} = \inner{w - w^*}{\nabla f(w)} \Rightarrow \\ \Rightarrow 
        \cos{\alpha} &= \frac{\inner{w - w^*}{\nabla f(w)}}{\nbr{w - w^*}_2\nbr{\nabla f(w)}_2} \ge m\frac{\nbr{w - w^*}_2}{\nbr{\nabla f(w) - \nabla f(w^*)}_2} \ge \frac{m}{L} \numberthis \label{eq:angle_bound}
    \end{align*}
    Now, notice that left-hand side of~\ref{eq:set_bound} converges to $1$ when $\nbr{w}_2 \to \infty$. Thus, for any sequence $w_i$, such that $\nbr{w_i}_2 \to \infty$
    \[ \forall \epsilon\ \exists M(\epsilon) : \abr{\frac{\nbr{w}_2}{\nbr{w - w^*}_2} - 1} < \epsilon\ \forall w: \nbr{w}_2 > M(\epsilon)\]
    Finally, let's choose $\epsilon = 1, M(1) \equiv M_1, \eta \le \frac{2m}{L(1 + \epsilon)} = \frac{m}{L}$, then
    \[
        \frac{\nbr{w}_2}{\nbr{w - w^*}_2} < 1 + \epsilon = 2 \le 
        \frac{2}{\eta}\frac{m}{L} \le \frac{2}{\eta}\cos{\alpha}\ \forall w: \nbr{w}_2 > M_1
    \]
    So, we just proved that for all sufficiently big $w$, inequality~\ref{eq:set_rev} holds, which means that 
    \[ \nbr{w}_2 > M_1 \Rightarrow w \notin S \Rightarrow \sup_{w\in S}\nbr{w}_2 \le M_1 < \infty \]
\end{proof}

\textbf{Theorem~\ref{th:main-conv}}.{ \textit{
    For any strongly convex function (convex if Conjecture~\ref{conj} is true) with Lipschitz gradient $f(w)$, LARS sequence of updates with fixed learning rate $w_{k+1} = w_k - \eta\frac{\nbr{w_k}_2}{\nbr{\nabla f(w_k)}_2}\nabla f(w_k)$ does not diverge to infinity, i.e. 
    \[ \exists M > 0\ \text{such that}\ \nbr{w_k - w^*}_2 \le M \ \forall k \]
    where $w^* = \argmin_w f(w)$ and $\eta$ is chosen small enough.
}}

\begin{proof}
In this proof we will use Lemma~\ref{lm:set} with constant bounding set $S$ denoted with $M_0$. Let's notice that for any 
\begin{align}\label{eq:1}
    w \in S \Rightarrow \nbr{w - w^*}_2 \le \nbr{w}_2 + \nbr{w^*}_2 \le M_0 + \nbr{w^*}_2 \equiv M_1
\end{align}
Moreover, one step deviation from set $S$ is also bounded, i.e. for any $w^+ \equiv w - \eta\frac{\nbr{w}_2}{\nbr{\nabla f(w)}_2}\nabla f(w)$
\begin{align}\label{eq:2}
    \sup_{w \in S}\nbr{w^+}_2 = \sup_{w \in S}\nbr{w - \eta\frac{\nbr{w}_2}{\nbr{\nabla f(w)}_2}\nabla f(w)}_2 \le  
    \sup_{w \in S}(1 + \eta)\nbr{w}_2 = (1 + \eta)M_1 \equiv M_2
\end{align}
Finally, notice that by definition of set $S$ we have
\begin{align}\label{eq:3}
    w \notin S \Rightarrow \nbr{w^+ - w^*}_2 \le \nbr{w - w^*}_2
\end{align}
Combining statements~\ref{eq:1},~\ref{eq:2} and~\ref{eq:3} we can now show that
\[ w_k \in S \Rightarrow \nbr{w_t - w^*}_2 \le \max\cbr{M_1, M_2} = M_2\ \forall t \ge k  \]
Indeed, if $w_{k+1} \in S \Rightarrow \nbr{w_{k+1} - w^*}_2 \le M_1$. If $w_{k_1} \notin S \Rightarrow  \nbr{w_{k+1} - w^*}_2 \le M_2$. And after that $w_{k+2}$ can only be closer to optimum than $w_{k+1}$ because~\ref{eq:3} is true. Thus, the distance will only decrease for all next steps until $w_t \in S$ again.

To conclude the proof notice that if $w_0 \in S \Rightarrow \nbr{w_k - w^*}_2 \le M_2$. If $w_0 \notin S$ than using~\ref{eq:3} again we see that all next steps will be closer to optimum, until $w_k \in S$. Thus, we have shown that
\[ \nbr{w_k - w^*}_2 \le \max\cbr{M_2, \nbr{w_0 - w^*}_2} \equiv M \ \forall k  \]
which concludes the proof.
\end{proof}

\textbf{Theorem~\ref{th:rate-fixed}}.\textit{
    Suppose that function $f : \R^n \to \R$ is strongly convex (convex if Conjecture 1 is true) and differentiable, and that its gradient is Lipschitz continuous with constant $L > 0$.
    Under the additional assumption that $\exists M_2 > 0 : \min_k\nbr{w_k}_2 \ge M_2$, the LARS sequence of updates with fixed learning rate $\eta$ satisfies the following inequality for each $k$
    \[ \min_k\sbr{f(w_k)} - f(w^*) \le \frac{C_1}{k\eta} + \eta^2 C_2 \]
    with
    \[ C_1 = \frac{LM_1}{2M_2}, C_2 = \frac{L\rbr{\nbr{w^*}_2 + M_1}^2}{2} \]
    with $f(w^*) = \min_w f(w)$ and $M_1$ being the constant from Theorem~\ref{th:main-conv}. }
\begin{proof}

From Lipschitz continuity follows that 
\[ \nbr{\nabla f(w_k)}_2 = \nbr{\nabla f(w_k) - \nabla f(w^*)}_2 \le L\nbr{w_k - w^*}_2 \le LM_1\ \forall k \] 
Moreover, from the reverse triangle inequality we have 
\[ \nbr{w_k}_2 - \nbr{w^*}_2 \le \nbr{w_k - w^*}_2 \Rightarrow \nbr{w_k}_2 \le \nbr{w^*}_2 + M_1 \equiv M_3 \] 
Now, let's denote $f(w_k) \equiv f_k, \nabla f(w_k) \equiv g_k, f(w^*) \equiv f^*$ then
\begin{align*}
    f_{k+1} &\le f_k + \inner{g_k}{w_{k+1} - w_k} + \frac{L}{2}\nbr{w_{k+1} - w_k}_2^2 = f_k - \eta\frac{\nbr{w_k}_2}{\nbr{g_k}_2}\nbr{g_k}_2^2 + \frac{\nbr{w_k}^2_2}{\nbr{g_k}^2_2}\frac{\eta^2 L}{2}\nbr{g_k}_2^2 \le \\
    &\le f^* + \inner{g_k}{w_k - w^*} - \eta\nbr{w_k}_2\nbr{g_k}_2 + \frac{\eta^2 L}{2}\nbr{w_k}_2^2 = \\ &=
    f^* + \frac{1}{2\eta}\frac{\nbr{g_k}_2}{\nbr{w_k}_2}\rbr{2\eta \frac{\nbr{w_k}_2}{\nbr{g_k}_2}\inner{g_k}{w_k - w^*} - \eta^2\nbr{w_k}_2^2 + \nbr{w_k - w^*}_2^2 - \nbr{w_k - w^*}_2^2} - \\ &\hspace{25px} - \frac{\eta}{2}\nbr{w_k}_2\nbr{g_k}_2 + \frac{\eta^2 L}{2}\nbr{w_k}_2^2 = \\
    &= f^* + \frac{1}{2\eta}\frac{\nbr{g_k}_2}{\nbr{w_k}_2}\rbr{\nbr{w_k - w^*}^2_2 - \nbr{w_{k+1} - w^*}_2^2} - \frac{\eta}{2}\nbr{w_k}_2\nbr{g_k}_2 + \frac{\eta^2 L}{2}\nbr{w_k}_2^2 \le \\ &\le 
    f^* + \frac{LM_1}{2\eta M_2}\rbr{\nbr{w_k - w^*}^2_2 - \nbr{w_{k+1} - w^*}_2^2} + \frac{\eta^2 M_3^2 L}{2}
\end{align*}

Summing up inequalities for each $k$ we achieve the final result:
\begin{align*} 
    \min_i\sbr{f_{i}} - f^* &\le \frac{1}{k}\sum_{i=1}^k\rbr{f_i - f^*} \le \frac{1}{k}\frac{LM_1}{2\eta M_2}\rbr{\nbr{w_0 - w^*}^2_2 - \nbr{w_{k+1} - w^*}_2^2} + \frac{\eta^2 M_3^2 L}{2} \le \\ &\le \frac{1}{k\eta}\frac{LM_1}{2M_2}\nbr{w_0 - w^*}^2_2 + \eta^2 \frac{L\rbr{\nbr{w^*}_2 + M_1}^2}{2}
\end{align*}

\end{proof}

\textbf{Theorem~\ref{th:rate-decay}}.\textit{
    Under the assumptions of Theorem~\ref{th:rate-fixed}, LARS converges to the optimal value with any decaying learning rate sequence $\eta_k$ satisfying 
    \[ \sum_{k=1}^{\infty}\eta_k = \infty, \sum_{k=1}^{\infty}\eta_k^2 < \infty \]
}
\begin{proof}

We will use the same notation as used in Theorem~\ref{th:rate-fixed}.

\begin{align*}
    \nbr{w_{k+1} - w^*}_2^2 &= \nbr{w_k - w^* - \eta_k \frac{\nbr{w_k}_2}{\nbr{g_k}_2}g_k}_2^2 = \\ &= \nbr{w_k - w^*}_2^2 - 2\eta_k\frac{\nbr{w_k}_2}{\nbr{g_k}_2}\inner{w_k - w^*}{g_k} + \eta_k^2\nbr{w_k}_2^2 \le \\ &\le \nbr{w_k - w^*}_2^2 - 2\eta_k\frac{\nbr{w_k}_2}{\nbr{g_k}_2}\rbr{f_k - f^*} + \eta_k^2\nbr{w_k}_2^2
\end{align*}
where in the last inequality we used the definition of convexity:
\[ f^* \ge f_k + \inner{g_k}{w^* - w_k} \Leftrightarrow -\inner{g_k}{w_k - w^*} \le -\rbr{f_k - f^*} \] 

Applying this inequality recursively we get
\begin{align*}
    0 \le \nbr{w_{k+1} - w^*}_2^2 &\le \nbr{w_1 - w^*}_2^2 - \sum_{i=1}^k\sbr{2\eta_i \frac{\nbr{w_i}_2}{\nbr{g_i}_2}\rbr{f_i - f^*}} + \sum_{i=1}^k\sbr{\eta_i^2\nbr{w_i}_2^2} \Rightarrow \\ \Rightarrow \rbr{\min_i\sbr{f_i} - f^*}&\sum_{i=1}^k\sbr{2\eta_i \frac{\nbr{w_i}_2}{\nbr{g_i}_2}} \le \nbr{w_1 - w^*}_2^2 + \sum_{i=1}^k\sbr{\eta_i^2\nbr{w_i}_2^2} \Rightarrow \\ \Rightarrow
    \min_i\sbr{f_i} - f^* &\le \frac{\nbr{w_1 - w^*}_2^2 + \sum_{i=1}^k\sbr{\eta_i^2\nbr{w_i}_2^2}}{2\sum_{i=1}^k\sbr{\eta_i \frac{\nbr{w_i}_2}{\nbr{g_i}_2}}} \le \frac{M_1^2 + M_3^2\sum_{i=1}^k\eta_i^2}{2\frac{M_2}{M_1L}\sum_{i=1}^k\eta_i}\numberthis \label{eq:decay-bound}
\end{align*}

Noticing that the numerator converges to a constant value as $k \to \infty$ and denominator grows without bound we conclude the proof.


\end{proof}

\end{document}